\documentclass[wcp]{jmlr}

\jmlrvolume{
}
\jmlryear{
}
\jmlrworkshop{
}

\usepackage{times}
\usepackage{helvet}
\usepackage{courier}
\usepackage{times} 
\usepackage{graphicx}
\usepackage{algorithm}
\usepackage{algorithmic}
\usepackage{amssymb,amsmath,bm}
\usepackage{pgfplots}
\pgfplotsset{compat=newest}
\usepgfplotslibrary{groupplots}
\usepackage{pgfplotstable}
\usepackage{booktabs}
\usepackage{colortbl}

\newcommand{\Rmnum}[1]{\expandafter\@slowromancap\romannumeral #1@}
\providecommand{\abs}[1]{\lvert#1\rvert}

\newcommand{\vc}[1]{{\pmb{#1}}}
\newcommand{\sign}{\operatorname{sign}}
\ifx\BlackBox\undefined
\newcommand{\BlackBox}{\rule{1.5ex}{1.5ex}}
\fi
\ifx\proof\undefined
\newenvironment{proof}{\par\noindent{\bf Proof\ }}{\hfill\BlackBox\\[1mm]}
\fi
\ifx\theorem\undefined
\newtheorem{theorem}{Theorem}
\fi
\ifx\lemma\undefined
\newtheorem{lemma}[lemma]{Lemma}
\fi
\ifx\definition\undefined
\newtheorem{definition}[theorem]{Definition}
\fi

\title[]{Totally Corrective Boosting with Cardinality Penalization}

 \author{\Name{Vasil S. Denchev} \Email{denchev@google.com}\\
 \addr Google, USA
 \AND
 \Name{Nan Ding} \Email{dingnan@google.com}\\
 \addr Google, USA
 \AND
 \Name{Shin Matsushima} \Email{shin\_matsushima@mist.i.u-tokyo.ac.jp}\\
  \addr University of Tokyo, Japan
  \AND
  \Name{S.V.N. Vishwanathan} \Email{vishy@ucsc.edu}\\
  \addr University of California, Santa Cruz, USA
  \AND
  \Name{Hartmut Neven} \Email{neven@google.com}\\
  \addr Google, USA
 }

\begin{document}

\maketitle

\begin{abstract}
We propose a totally corrective boosting algorithm with explicit cardinality regularization. The resulting combinatorial optimization problems are not known to be efficiently solvable with existing classical methods, but emerging quantum optimization technology gives hope for achieving sparser models in practice. In order to demonstrate the utility of our algorithm, we use a distributed classical heuristic optimizer as a stand-in for quantum hardware. Even though this evaluation methodology incurs large time and resource costs on classical computing machinery, it allows us to gauge the potential gains in generalization performance and sparsity of the resulting boosted ensembles. Our experimental results on public data sets commonly used for benchmarking of boosting algorithms decidedly demonstrate the existence of such advantages. If actual quantum optimization were to be used with this algorithm in the future, we would expect equivalent or superior results at much smaller time and energy costs during training. Moreover, studying cardinality-penalized boosting also sheds light on why unregularized boosting algorithms with early stopping often yield better results than their counterparts with explicit convex regularization: Early stopping performs suboptimal cardinality regularization. The results that we present here indicate it is beneficial to explicitly solve the combinatorial problem still left open at early termination.
\end{abstract}

\section{Introduction}\label{intro}
Boosting algorithms execute the following protocol in each iteration \citep{FreSch97,Freund98}. The algorithm provides a distribution $\vc{u} \in \mathbb{R}_{\geq 0}^m$ on a given set of $m$ training examples. Then an {\em oracle} provides a {\em weak hypothesis} from some hypothesis class and the distribution is updated.  At the end, the algorithm arrives at a linear combination $\vc{w}^\star \in \mathbb{R}_{\geq 0}^n$ of weak hypotheses, where $n$ is the number of all possible weak hypotheses. One can view boosting as a zero-sum game between a row and a column player \citep{Warmuth06}. Each possible hypothesis provided by the oracle is a column of an underlying game matrix $\mathcal{H}  \in \{-1,1\}^{m\times n}$ that represents the entire hypothesis class available to the oracle.  The training examples correspond to rows $\mathcal{H}_{i:}$ of this matrix. $y_{i} \mathcal{H}_{i:} \vc{w}$ is the {\em margin} of example $i$ with label $y_i$ w.r.t. the linear combination $\vc{w}$ of the weak hypotheses in $\mathcal{H}$. From an optimization perspective, one can view different boosting algorithms as \emph{column generation} approaches \citep{Bennett02} for solving regularized risk minimization:
\begin{equation}\label{convex_primal}
\min_{\vc{w}} \sum_{i=1}^{m} l(y_i\mathcal{H}_{i:} \vc{w}) + \nu\Omega(\vc{w}) \quad \mbox{ s.t. } \vc{w} \succeq \vc{0} \enspace ,
\end{equation}
where $l(\cdot)$ and $\Omega(\cdot)$ denote loss function and regularization typically assumed to be convex. The number of columns in $\mathcal{H}$ is in principle unbounded, but it is finite for all finite data sets in practice.

Popular boosting algorithms commonly choose the $\ell_{1}$ penalty as regularization, and hence, we expect that the minimizer $\vc{w}^\star$ is sparse. In fact, good boosting algorithms come with theoretical proofs showing that a $\tilde{\vc{w}}$ that is $\epsilon$-close to $\vc{w}^\star$ contains $O(1/\epsilon^2)$ non-zero entries \citep{Warmuth08}. At this point it is worthwhile to make a distinction between \emph{corrective} and \emph{totally corrective} boosting algorithms. Corrective boosting algorithms such as AdaBoost only update one coordinate of $\vc{w}$ at every iteration. On the other hand, totally corrective algorithms such as ERLPBoost \citep{Warmuth08} update all active coordinates at every iteration. Namely, in the $t$-th iteration they solve a $t$-dimensional optimization problem. Experimental evidence suggests that totally corrective algorithms produce significantly sparser solutions than corrective algorithms. This leads to the central question of this paper: What if we explicitly enforced on $\vc{w}$ a sparsity-inducing \emph{cardinality penalization} (CP), also known as the $\ell_0$ pseudo-norm \citep{Bach11}, in order to produce solutions with as few active weak hypotheses as possible? This requires solving
\begin{equation}
\min_{\vc{w}} \sum_{i=1}^{m} l(y_i \mathcal{H}_{i:} \vc{w}) + \nu\vc{1}^\top\vc{w} +\lambda\bm{\mbox{card}}(\vc{w}) \quad \mbox{ s.t. }  \vc{w} \succeq \vc{0} \enspace ,
\end{equation}
where $\bm{\mbox{card}}(\vc{w}) =\left|\{i : w_i \neq 0\}\right|$ counts the number of nonzero elements in $\vc{w}$, and the $\ell_{1}$ term with a negligibly small but nonzero $\nu$ avoids ill-posed optimization problems.

The motivation for our work is twofold. First, the emergence of commercial quantum optimization technology offers hope that in the near future one will be able to directly solve problems involving CP.\footnote{For now we only experiment with a costly heuristic solver distributed over classical machines.} Second, there is a definite need to deploy high-accuracy classifiers in resource-constrained environments ranging from mobile and wearable devices (see Subsection~\ref{glass}) to micro drones and deep-space probes. Because latency, power usage, and processor cycles at prediction time are critical in such applications, cardinality penalization in boosted ensembles can come with disproportionate payoffs in terms of the practical results that suddenly become possible.

The paper is organized as follows. Section~\ref{related_work} reviews background. In Sect.~\ref{tqboost} we first study the impact of CP on boosting by deriving the Lagrange dual function and then state and discuss the \emph{TotalQBoost} algorithm. In Sect.~\ref{expe_setup},~\ref{results}, and~\ref{conclusion} we present the experimental setup, list results, and conclude.

\section{Background}\label{related_work}
{\bf Boosting}: \citet{CShen10,CShen13}, inspired by the seminal work of \citet{Mason99}, cast boosting as general regularized risk minimization over the space $\mathcal{H}$ of all possible columns. Boosting algorithms can in principle incorporate all sensible loss functions and regularizers provided the Lagrange duals of the corresponding primal problems can be derived and analyzed. Lagrange duality is a useful tool in the theoretical treatment because the idea of \emph{column generation} (CG) is intimately tied to it. \citet{Bennett02} established the CG view of boosting and showed how it can be used in convergence analysis and derivation of termination criteria for the resulting practical algorithms. Moreover, when strong duality holds, the algorithm designer can opt for putting the burden of optimization on the dual rather than the primal problem if practical benefits are seen. The regularized risk minimization perspective emphasizes the similarity between boosting and kernel methods. Boosting, via linear combinations of weak hypotheses from $\mathcal{H}$, also effectively maps the original data features onto a high-dimensional space, thus enabling the construction of non-linear decision surfaces in the original space.

{\bf Cardinality penalization}: As the ultimate sparsity promoter, CP has been considered many times in the context of various machine learning problems. However, researchers have avoided dealing with it directly because it leads to combinatorial problems not known to be solvable in polynomial time by classical computing. \citet{Pilanci} propose a convex relaxation for CP and use it in recovering sparse probability measures over moment constraints and clustering via sparse Gaussian mixtures; \citet{Borwein} propose a convex relaxation consisting of entropic regularization of the zero metric and study its convergence to sparse solutions; \citet{TongZhang2012} survey existing statistical theory results on non-convex regularization and construct a general theoretical framework for them; \citet{Bach11} present an extensive treatise on optimization tools and techniques dedicated to sparsity-inducing penalties.

{\bf Quantum optimization}: D-Wave Systems, Inc. develops quantum optimization hardware. The D-Wave machines are not universal quantum computers as their specialized hardware is designed to take advantage of limited quantum effects at finite temperature on a restricted Ising model \citep{Brush}. This is a departure from the idealistic adiabatic quantum algorithm of \citet{Farhi2001} but to date is the only practical design to ever achieve ${\sim} 1000$ functional qubits. Up to now there have been several studies, e.g. \citet{Lanting14} and \citet{Boixo14}, thoroughly characterizing the underlying quantum effects that are regarded by the wider quantum computing community as necessary for achieving computational speedups over classical algorithms. Simultaneously, there are ongoing efforts to systematically characterize the actual speedup that the D-Wave machines can offer \citep{Ronnow14}. There is also prior machine learning research using quantum optimization: robust classification with a non-convex loss function that is compatible with the D-Wave architecture \citep{Denchev12}; a heuristic boosting algorithm with CP \citep{Neven12}; and training a detector for cars in digital images with a D-Wave optimizer \citep{CarDetector}. In general, the potential advantage of adiabatic quantum optimization (AQO) over classical algorithms is best understood by comparison with its closest classical analogue---simulated thermal annealing \citep{Denchev13}. As illustrated in Fig.~\ref{fig:quantum_annealing}, AQO has at its disposal \emph{quantum tunneling}, which is a fundamentally non-classical resource for escaping local minima in hard optimization problems. While thermal annealing may get trapped in a deep local minimum and not have enough thermal energy to jump over the tall potential barrier, quantum annealing has a better chance of reaching the global minimum with the help of tunneling.

\begin{figure}[t!]
\centering
\includegraphics[width=5cm, height=5.5cm]{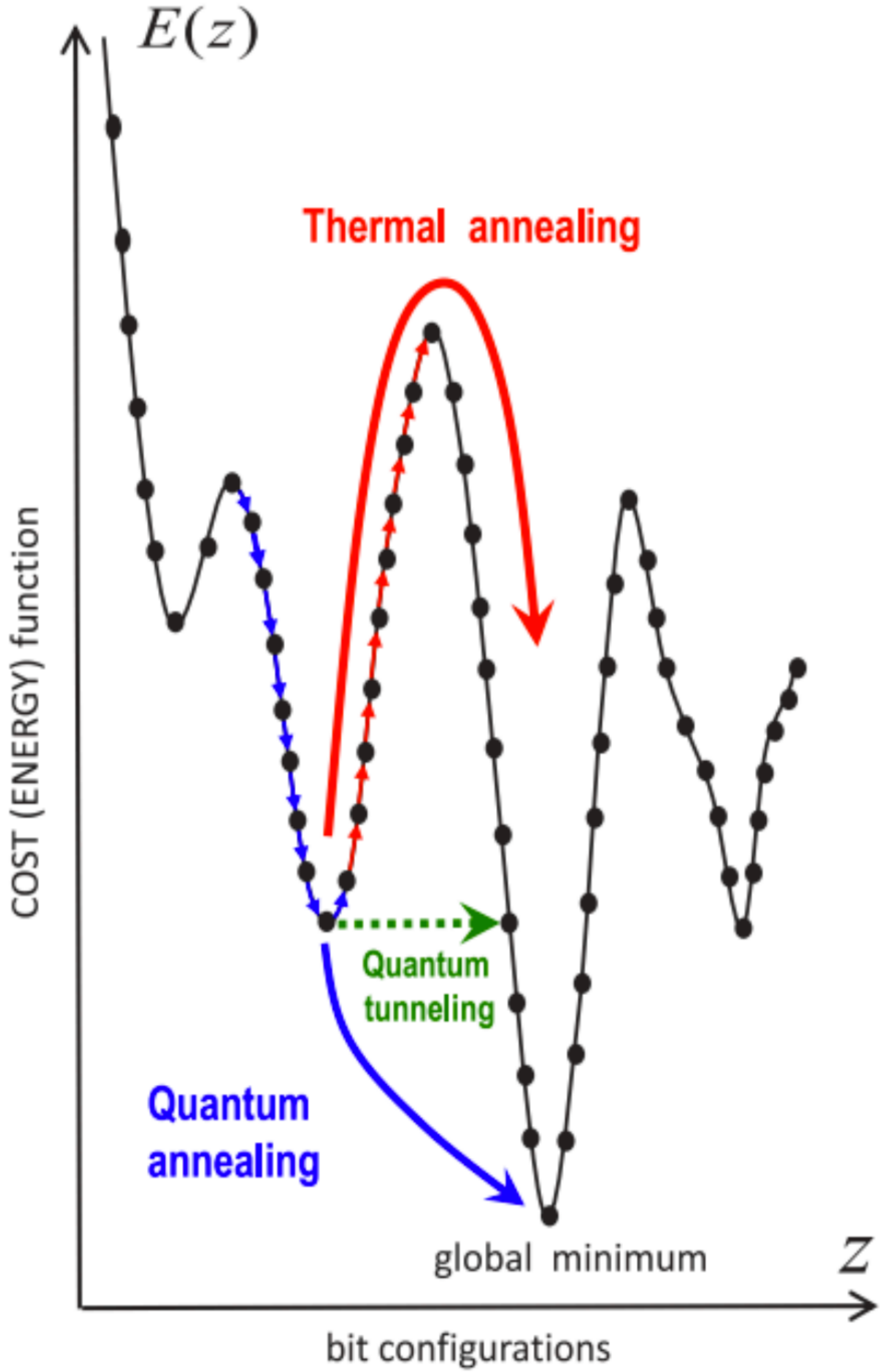}
\caption{Reprinted from \citet{Smelyanskiy2012}. A possible optimization landscape where thermal annealing initially gets trapped in a local minimum. Its only escape route is over the barrier via thermal fluctuations. In contrast, quantum annealing can escape the local minimum via quantum tunneling.}
\label{fig:quantum_annealing}
\end{figure}

\section{Totally Corrective Boosting with CP}\label{tqboost}
In Subsection~\ref{card_penal} we first set the stage for cardinality-penalized boosting by analyzing the Lagrange dual function. Surprisingly, we find that \emph{the dual is unaffected by CP in the primal}. The fascinating failure of CP to propagate its influence to the dual provides the verification that sample weights given by the dual variables are valid, and CG can work correctly in the presence of CP. \citet{Borwein} similarly observe that going from primal to dual erases all information about the behavior of CP. This counter-intuitive result requires investigation from different angles in order to be fully understood. To that end, in Appendix A and B we provide two additional arguments independently arriving at the same conclusion.

Subsection~\ref{column_generation} presents some more detailed background on CG in order to facilitate the understanding of the main algorithm. In Subsection~\ref{totalqboost_algorithm} we look at the duality gap created by the CP term in the primal problem and gain useful insights for the design of the boosting algorithm. Finally, in Subsections~\ref{early_stopping} and~\ref{discrete_variables} we relate early stopping to CP and discuss practical issues stemming from conversions between continuous and discrete variables necessitated by the architecture of quantum hardware.

\subsection{Lagrange Theory for Boosting with CP}\label{card_penal}
The primal problem for $\ell_1$- and cardinality-regularized risk minimization is
\begin{align}
  &\min_{\vc{w}, \vc{\gamma},\vc{\eta}} \sum_{i=1}^m l(\gamma_i) +
  \nu\vc{1}^\top\vc{w} + \lambda
  \bm{\mbox{card}}(\vc{\eta}) \label{primal_with_l0} \\
&\mbox{s.t.} \quad
  \gamma_i = y_i \mathcal{H}_{i:} \vc{w} \mbox{ (for } i = 1,\ldots,m\mbox{)},\quad
  \vc{\eta} = \vc{w}, \quad \vc{w} \succeq \vc{0} \mbox{ ,}\nonumber
\end{align}
where $l(\cdot)$ is a loss function; a column $\mathcal{H}_{:j}$ stores the $\{-1,1\}$ responses of the $j$-th weak classifier on all training examples; and dummy variables $\vc{\gamma}$ and $\vc{\eta}$ allow for arriving at a meaningful dual. The Lagrangian, subject to $\vc{p} \succeq \vc{0}$, is
\begin{equation}
  L= \left(\nu \vc{1}^\top + \lambda\vc{s}^\top - \vc{u}^\top \mbox{diag}\left(\vc{y}\right)\mathcal{H} - \vc{p}^\top\right)\vc{w} +   \vc{u}^\top\vc{\gamma}\nonumber\\
  + \sum_{i=1}^m l\left(\gamma_i\right) - \lambda
  \left(\vc{s}^\top\vc{\eta} - \bm{\mbox{card}}(\vc{\eta})\right)
  \enspace .
\end{equation}
To find the infimum over primal variables $\vc{w}$ and $\vc{\gamma}$, we set derivatives to zero and obtain:
\begin{align}
&\vc{u}^\top \mbox{diag}\left(\vc{y}\right)\mathcal{H} \preceq \nu\vc{1}^\top + \lambda\vc{s}^\top\label{constraints}\\
&u_i = -l'\left(\gamma_i\right) , \forall i \enspace .\label{dual_variables}
\end{align}
Equation \eqref{dual_variables} means that a dual variable $u_i$ is exactly the
negative derivative of the loss function $l$ at margin $\gamma_i$. With
\eqref{constraints} satisfied and $l^*\left(-u_i\right)$ and
$\bm{\mbox{card}}^*(\vc{s})$ denoting the Fenchel conjugates of
$l\left(\gamma_i\right)$ and $\bm{\mbox{card}}(\vc{w})$, respectively,
\begin{equation}
\inf_{\vc{w},\vc{\gamma},\vc{\eta}} L = -\sum_{i=1}^m l^*\left(-u_i\right) - \lambda\bm{\mbox{card}}^*(\vc{s})\enspace .
\end{equation}
Hence, the dual is
\begin{equation}\label{dual2}
\min_{\vc{u},\vc{s}} \sum_{i=1}^{m} l^*\left(-u_i\right) + \lambda\bm{\mbox{card}}^*(\vc{s}) \quad \mbox{s.t. } \eqref{constraints}\enspace .
\end{equation}
It appears $\bm{\mbox{card}}(\vc{\eta})$ from the primal impacts the dual via $\bm{\mbox{card}}^*(\vc{s})$ and
$\vc{s}$. However, Lemma~\ref{fenchel_conjugate} shows $\bm{\mbox{card}}^*(\vc{s}) = 0$ and $\vc{s} = \vc{0}$.
\begin{lemma}\label{fenchel_conjugate}
  The Fenchel conjugate of CP is
\begin{equation*}
  \bm{\mbox{card}}^*(\vc{s}) = \sup_{\vc{w} \in
    \mathbb{R}^n}\,\vc{s}^\top \vc{w} -\bm{\mbox{card}}(\vc{w}) =
	\begin{cases}
		0 & \mbox{ for } \vc{s} = \vc{0}\\
		\infty & \mbox{ otherwise}
	\end{cases}
\end{equation*}
\end{lemma}
\begin{proof}
For $\vc{s} = \vc{0}$, $\displaystyle \bm{\mbox{card}}^*(\vc{s}) = \sup_{\vc{w}}\, -\bm{\mbox{card}}(\vc{w}) = 0$. For any $\vc{s} \neq \vc{0}$, since $\bm{\mbox{card}}(\vc{w}) \leq n$, $\displaystyle \bm{\mbox{card}}^*(\vc{s}) = \sup_{\vc{w}}\,\vc{s}^\top \vc{w} -\bm{\mbox{card}}(\vc{w}) \geq \sup_{\vc{w}}\, \vc{s}^\top \vc{w} - n = \infty$.
\end{proof}

Equation~\eqref{dual2} together with Lemma~\ref{fenchel_conjugate} prove that the dual problem is oblivious to CP:
\setcounter{theorem}{0}
\begin{theorem}\label{identical_duals}
Equations~\eqref{primal_with_l0} and~\eqref{convex_primal} with $\Omega(\vc{w}) = \vc{1}^\top\vc{w}$ have identical Lagrange duals.
\end{theorem}

\subsection{Column Generation}\label{column_generation}
As discussed in Sect.~\ref{intro}, $\mathcal{H}$ is assumed to be finite but very large, and thus, it is impractical to directly solve regularized risk minimization over $\mathcal{H}$ even when convexity holds. The role of CG, which serves as the fundamental optimization framework for boosting algorithms, is to solve the full optimization problem over $\mathcal{H}$ in incremental steps by always considering only a small subset of active columns $H \subset \mathcal{H}$ and augmenting $H$ one column at a time. When convex duality holds, only columns with nonzero weights in the optimal solution $\vc{w}^\star$ need to ever be explicitly considered as part of $H$, so most columns in $\mathcal{H}$ can safely be ignored throughout the entire algorithm \citep{Bennett02}.

Hence, at iteration $t$ a $t$-dimensional problem, known as \emph{restricted master problem} (RMP), is optimized. After solving an RMP, the algorithm generates the next RMP by finding a column from $\mathcal{H}$ that violates a constraint in \eqref{constraints}. When convex duality holds we are guaranteed that for all columns that are already in $H$, solving the current RMP also satisfies the corresponding dual constraints in \eqref{constraints}. Termination occurs when no such violation by more than $\epsilon > 0$ can be found anymore:
\begin{equation}\label{l1_termination}
\nexists i \quad \mbox{such that} \quad \vc{u}^\top \mbox{diag}\left(\vc{y}\right)\mathcal{H}_{:i} > \nu + \epsilon \mbox{ ,}
\end{equation}
where $\nu$ is the $\ell_1$ regularization coefficient.

From this perspective, solving primal RMPs in successive iterations corresponds to solving increasingly tightened relaxations of the full dual problem over $\mathcal{H}$: Violated dual constraints for columns in $\mathcal{H} - H$ are brought into consideration and satisfied one at a time. Upon termination, $\vc{w}^\star$ implicitly gives zero weights to columns still in $\mathcal{H} - H$ because their dual constraints are satisfied to within $\epsilon$. Optimality follows from the fact that at termination we have $\epsilon$-approximate primal and dual feasibility with equal objective values on the full problem over $\mathcal{H}$.

\subsection{TotalQBoost}\label{totalqboost_algorithm}
In this paper we insist on directly optimizing the primal cardinality-penalized RMP at each boosting iteration. However, CP makes the problem non-convex and destroys strong duality. With that, we do not have any guarantee that CG will converge to the global minimum of the cardinality-penalized risk minimization problem even if all RMPs at successive boosting iterations are solved to optimality. We also obtain a technical complication from the fact that solving a non-convex primal RMP to optimality does not guarantee the satisfaction of all dual constraints \eqref{constraints} corresponding to columns in $H$. Even so, CG in conjunction with solving individual RMPs to optimality enables us to take steps and overcome local minima in the high-dimensional non-convex space of \eqref{primal_with_l0}. Even if we fail to reach the global minimum of \eqref{primal_with_l0}, solving non-convex problems along the way is highly non-trivial optimization work that results in better boosted ensembles than what is possible with convex methods.

When strong duality does not hold, a quantity of interest is the duality gap, i.e. the difference between the optimal values of the primal and dual problems. Fortunately, due to the fact that the dual does not change under CP, we can determine a bound on the duality gap solely on the basis of the impact of the CP term in the primal. This also gives insight into the design of the algorithm that follows.
\begin{theorem}
The duality gap $\delta$ between primal \eqref{primal_with_l0} and dual \eqref{dual2} is $\delta \geq \lambda\bm{\hbox{card}}(\vc{w})$ with equality holding for $\lambda$ such that
\begin{equation}\label{equality_holding}
\arg\min_{\vc{w} \succeq \vc{0}} \sum_{i=1}^{m} l(y_i \mathcal{H}_{i:} \vc{w}) + \nu\vc{1}^\top\vc{w} +\lambda\bm{\mbox{card}}(\vc{w})
=\arg\min_{\vc{w} \succeq \vc{0}} \sum_{i=1}^{m} l(y_i \mathcal{H}_{i:} \vc{w}) + \nu\vc{1}^\top\vc{w} \enspace .
\end{equation}
\end{theorem}
\begin{proof}
By Theorem~\ref{identical_duals}, the dual of the primal with CP is the same as the dual of the primal without CP. Moreover, because we are assuming a convex loss function, strong duality holds for the latter primal problem and its duality gap is zero. Hence, if the optimal solution to the primal with CP (left-hand side of \eqref{equality_holding}) gives the best empirical risk under $\ell_1$ regularization only (right-hand side of \eqref{equality_holding}), then $\delta = \lambda\bm{\mbox{card}}(\vc{w})$ because this is the difference in objective values at the common minimum. Alternatively, if $\lambda$ is so large that the optimal solution to the primal with CP no longer gives the best empirical risk under $\ell_1$ regularization, then $\delta >\lambda\bm{\mbox{card}}(\vc{w})$.
\end{proof}

Figure~\ref{fig:duality_gap}, \emph{Left} shows the case of weak CP in which the location of the global minimum is the same as the minimum of the problem without CP. In this regime if the loss function $l(\cdot)$ is strictly convex, the term $\lambda\bm{\mbox{card}}(\vc{w})$ does not have any impact on the output of the boosting algorithm because the minimum of the $\ell_1$-regularized risk is unique. On the other hand, Fig.~\ref{fig:duality_gap}, \emph{Right} illustrates the case of strong CP, which causes the location of the global minimum to be no longer the same as the minimum of the convex primal. Consequently, the corresponding dual point may be suboptimal and infeasible, which allows for possible violations of dual constraints \eqref{constraints} corresponding to columns in $H$ whose weights are zeroed out by CP. In other words, by forcing some weights in the current RMP to zero, the CP term leaves the corresponding dual constraints just as violated as they were before adding these columns to $H$. Since these dual constraints are left violated by the current RMP, the weak learner oracle may try to offer the corresponding columns repeatedly in subsequent iterations. However, CG still proceeds in a well defined manner if we prevent the oracle from offering for the next iterations columns that have already been added to $H$ but have been forced to zero weights by the CP term when solving subsequent RMPs.

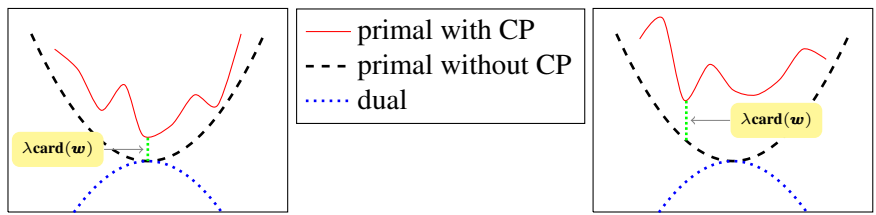
\begin{figure}[t!]
\centering
\tikzset{every mark/.append style={scale=0.5},every pin/.style={fill=yellow!50!white,rectangle,rounded corners=3pt,font=\tiny},small dot/.style={fill=black,circle,scale=0.3},arrow/.style={arrows=<->}}
\pgfplotsset{every axis/.append style={height=4.29cm,width=5.3cm}}
\begin{tikzpicture}
   \begin{axis} [name=plot1,ymin=-10,ymax=30,smooth,xtick=\empty,ytick=\empty,legend entries={primal with CP, primal without CP,dual}, legend style={cells={anchor=west}, legend pos=outer north east}]
    \addplot [color=red,mark=none] coordinates {(-4,22) (-3,18) (-2,10) (-1,15) (-0.2,5) (1,7) (2,13) (3,11) (4,25)};
    \addplot [color=black,mark=none,style=dashed,line width=1pt]{x^2};
    \addplot [color=blue,mark=none,style=dotted,line width=1pt]{-x^2};
    \addplot [color=green,mark=none,densely dotted,line width=1pt] coordinates {(0,4.6) (0,0)};
    \node[coordinate,pin={[pin edge={<-}]left:{$\lambda\bm{\mbox{card}}(\vc{w})$} } ] at (axis cs:-0.2,2.3) {};
  \end{axis}
  \begin{axis} [name=plot2,ymin=-10,ymax=30,smooth,at={($(plot1.east)+(4.06cm,0)$)},anchor=west,xtick=\empty,ytick=\empty]
    \addplot [color=red,mark=none] coordinates {(-4,24) (-3,28) (-2.1,12) (-1,19) (0,14) (1,13) (2,16) (3,22) (4,20)};
    \addplot [color=black,mark=none,style=dashed,line width=1pt]{x^2};
    \addplot [color=blue,mark=none,style=dotted,line width=1pt]{-x^2};
    \addplot [color=green,mark=none,densely dotted,line width=1pt] coordinates {(-2,11.85) (-2,4)};
    \node[coordinate,pin={[pin edge={<-}]right:{$\lambda\bm{\mbox{card}}(\vc{w})$} } ] at (axis cs:-1.8,8) {};
  \end{axis}
\end{tikzpicture}
 \caption{Stylized depictions of primal and dual functions for intuitive understanding. \emph{Left:} Global minimum of primal with CP corresponds to the best empirical risk under $\ell_1$ regularization only (small $\lambda$). In this case the duality gap is $\lambda\bm{\mbox{card}}(\vc{w})$; \emph{Right:} Global minimum of primal with CP is different from the best empirical risk under $\ell_1$ regularization only (large $\lambda$). The duality gap is greater than $\lambda\bm{\mbox{card}}(\vc{w})$.}
\label{fig:duality_gap}
\end{figure}

In this way, because the CP term forces successive RMPs to carry over violated dual constraints, the usual termination by dual constraint satisfaction of the $\ell_1$ regularizing condition \eqref{l1_termination} is abandoned. The CP term may simply be so strong that it could prevent some violated dual constraints from ever becoming satisfied. Even so, we can go on with CG iterations while blacklisting generated columns until the oracle is no longer able to offer any previously unseen columns. If the weak learner dictionary is of such a large size that termination does not occur even after many stagnant iterations of RMP solutions sticking to an old subset of generated columns while rejecting recent ones, then early stopping can be applied in order to quit solving primal RMPs that might have grown too large by that point. This is stated as Algorithm~\ref{alg:totalqboost}: \emph{TotalQBoost}.

\begin{algorithm}[h!]
\caption{TotalQBoost}\label{alg:totalqboost}
\begin{algorithmic}[1]
\REQUIRE Loss function $l(\gamma)$, $m$ training examples $\{(\vc{x}, y)\}$, dictionary of weak classifiers $\{h(\vc{x})\}$, regularization parameters $\nu$ and $\lambda$, tolerance $\epsilon$, maximum iterations $T$
\ENSURE Boosted ensemble $(H,\vc{w})^\star$
\medskip
\small{
	\STATE Initialize: $u_i ={\frac{1}{m}}$, $\forall i = 1, \ldots, m$; empty ensemble $(H,\vc{w})$; $t = 0$
   	\WHILE {$t < T$}
	    		\STATE Optimize members of $\{h(\vc{x})\}$ with current $\vc{u}$
			\IF {$\{h(\vc{x})\} - H = \emptyset$ \OR eq. \eqref{l1_termination} holds}
				\STATE \textbf{break}
			\ELSE
	    			\STATE Select weak classifier by most violated constraint in eq.
				\eqref{constraints}:\\$\hat{h}(\vc{x}) = \arg\max_{h(\vc{x}) \in \{h(\vc{x})\} - H} \sum_{i=1}^m u_{i}y_{i}h(\vc{x}_i)$
				\STATE Add $\hat{h}(\vc{x})$ as a new active column: $H = \{H, \hat{h}(\vc{x})\}$
			\ENDIF
			\STATE Optimize eq. \eqref{primal_with_l0} with $\{(\vc{x}, y)\}, H, \nu, \lambda$ over discrete variables $\dot{\vc{w}}$
			\STATE Update $\vc{w} = \dot{\vc{w}}$
			\STATE Let $\tilde{H} = \{h_k(\vc{x})\}$ for $k$ such that $\dot{w}_k > 0$ 
			\STATE Optimize $\ell_1$-regularized risk with $\{(\vc{x}, y)\}, \tilde{H}, \nu$, and tolerance $\epsilon$ over continuous variables $\tilde{\vc{w}}$
			\STATE Copy the elements of $\tilde{\vc{w}}$ to corresponding elements of $\vc{w}$
			\STATE Update $t = t + 1$ and $u_i = -l'(y_iH_{i:}\vc{w})$
	\ENDWHILE
	\STATE $(H,\vc{w})^\star = (\{h_k(\vc{x})\},\{w_k\})$ for $k$ such that $w_k > 0$ 
}
\end{algorithmic}
\end{algorithm}

\subsection{Early Stopping}\label{early_stopping}
Even without an explicit CP term in the primal, early stopping puts the final solution in a local minimum of the cardinality-penalized primal for all CP coefficients $\lambda > 0$. This is so because early stopping ensures all columns that could yet be generated are left with zero weights in the final ensemble. However, any point in the primal space with at least one zero-weighted coordinate is a local minimum for $\bm{\mbox{card}}(\vc{w})$.

Consider an unregularized CG algorithm such as AdaBoost, which is much simpler and in practice has been known to give better generalization than the various studied boosting algorithms with explicit convex regularization. While the lack of explicit regularization in AdaBoost is bound to eventually cause overfitting as more and more columns are generated, early stopping acts as cardinality penalization. Early stopping simply limits the maximum number of columns that are to be generated by the algorithm. However, the cardinality penalization that early stopping provides is suboptimal, as the first $T$ generated columns are unlikely to be the optimal set of $T$ columns for minimizing empirical risk. 

\subsection{Discrete Weight Variables}\label{discrete_variables}
TotalQBoost uses discrete weight variables $\dot{\vc{w}}$ in line 10 of Algorithm~\ref{alg:totalqboost} because the quantum optimization hardware is engineered as an interconnected collection of physical qubits, each of which is regarded as a binary variable at the computational level. Because the current D-Wave hardware has just ${\sim} 1000$ functional qubits, each element of $\dot{\vc{w}}$ is also restricted to use only a small number of bits in a fixed-point representation that implements the discrete variables $\dot{\vc{w}}$ via binary expansions using the underlying qubits. A practical issue stemming from converting from continuous to discrete variables is that one may lose the ability to represent weight configurations at or close enough to the lowest attainable objective value on continuous variables. A possible fix for this is to keep increasing the bit-depth and adjusting the range of elements of $\dot{\vc{w}}$ until the best non-CP empirical risk over continuous variables can be reached to within some tolerance. We can then run cardinality-penalized discrete optimization. However, even though the necessary bit-depth is not expected to be very high \citep{Neven12}, it might still be possible for this to result in a total number of binary variables that is too large to optimize effectively. In that case the number of iterations $T$ that TotalQBoost can perform is limited by the size of our discrete optimization facilities. Assuming that the variables in $\dot{\vc{w}}$ have enough bit-depth for correctly selecting the best subset of columns, we still have the chance to subsequently refine the weights of the selected columns to minimize $\ell_1$-regularized empirical risk. To that end, as specified by lines 12 and 13 of Algorithm~\ref{alg:totalqboost}, for the CP-selected columns we run with continuous weight variables $\tilde{\vc{w}}$ an off-the-shelf convex optimization algorithm such as \emph{L-BFGS-B} \citep{Morales2011}.

\section{Experimental Setup}\label{expe_setup}
We compare two variants of convex CG and three variants of cardinality-penalized CG per data set.

{\bf A. $\ell_1$-regularized CG (L1CG)}: This is the ordinary $\ell_1$-regularized CG. We perform multiple runs to $\epsilon$-convergence with a variety of regularization coefficients $\nu$ and record the performance of the converged ensembles.

{\bf B. Unregularized CG (UCG) with early stopping}: This is also $\ell_1$-regularized CG but with a negligibly small yet nonzero regularization coefficient $\nu$. Nonzero $\nu$ is needed for avoiding ill-posed optimization problems in the case of separable data. Early stopping is applied when too many columns are generated before reaching the termination criterion \eqref{l1_termination}. We perform a single run up to a maximum number of iterations $T$ and record the performance of all intermediate ensembles in order to gauge the benefit of early stopping at different times.

{\bf C. Cardinality-penalized CG (CPCG)}: We apply TotalQBoost with a negligibly small but nonzero $\ell_1$ regularization coefficient $\nu$ and a variety of CP coefficients $\lambda$ and record the performance of the ensembles at termination. Depending on the structure of the dictionary of weak classifiers, the weak learner may run out of useful columns before we reach the global minimum of the full primal problem with CP over all possible columns \eqref{primal_with_l0}. Thus, we do not have any guarantee of reaching the global minimum of the full problem. On the other hand, if a pre-set maximum number of iterations $T$ is reached, we apply early stopping.

{\bf D. Hot-started CPCG}: Hot-starting is done for the purpose of possibly getting to a better local minimum than the one reached by {\bf C.} We first perform early-stopping UCG for $T'$ iterations, then initialize CPCG with these $T'$ columns, and continue as per the TotalQBoost algorithm. Termination likely results in a different local minimum than the one that {\bf C.} produces for the same $\lambda$. Here, too, we repeat for a variety of CP coefficients $\lambda$ and record performance after termination of each.

{\bf E. Subset selection}: We take the columns generated in the first phase of {\bf D.}, solve the primal RMP corresponding to \eqref{primal_with_l0} only once for all different $\lambda$ values, and record the performance of the resulting ensembles. This experiment is computationally cheaper than {\bf C.} and  {\bf D.} because we solve only a single discrete optimization problem per $\lambda$ value. The purpose is to gauge the suboptimality of ensembles produced by {\bf B.} as well as the worthiness of the extra effort exerted by {\bf D.}

TotalQBoost as well as experiments {\bf C.}, {\bf D.}, {\bf E.} described above assume the existence of a discrete optimization method that solves to optimality the combinatorial problems arising from including direct CP in the learning problem. The main motivation for attempting to study a machine learning algorithm giving rise to  combinatorial problems lies in the recent advances in physical implementations of scalable quantum computing, as introduced in Sect.~\ref{related_work}. However, because a sufficiently large quantum machine was not available to us at the time of writing, we resorted to a distributed heuristic method attempting to solve classically the problems that would otherwise be delegated to quantum hardware. Hence, it is important to note that a potential weakness of the experimental results obtained in this manner is that we have no guarantee the cardinality-penalized optimization problems are solved to optimality. Consequently, in the actual comparison between the five experiments, the results ascribed to experiments {\bf C.}-{\bf E.} are probably not optimal. Nevertheless, any advantages manifested in the results of {\bf C.}-{\bf E.} can certainly be taken as evidence for the potential performance improvements to be delivered by cardinality-penalized totally corrective boosting whenever quantum optimization hardware becomes available to the machine learning practitioner.

\section{Results}\label{results}
The classical heuristic method used as a stand-in for quantum optimization hardware is Multistart Tabu Search \citep{Palubeckis2004}. We implemented a specialized distributed version of it and ran all experiments on all data sets over the course of a week on a collection of 1,800 commodity-grade machines with 16 cores each. Even with such a large amount of computational power we do not have optimality guarantees when solving the cardinality-penalized optimization problems, but some reasonable amount of tuning is done to the Tabu algorithm in order to ensure a basic level of confidence in its solutions.

The dictionary of weak classifiers is constructed as a collection of decision stumps, each taking one of the original features in a given data set. In this work we focus on studying the effect of CP, so the choice of loss function is inconsequential. Hence, in all experiments we use the most common loss function for boosting---exponential loss.\footnote{Prior work considering D-Wave, e.g. \citet{Denchev12}, was concerned with hardware-compatible loss functions. Here we abandon this requirement due to the recent development of the ``Blackbox compiler," first used by \citet{McGeoch}.} The optimization tolerance $\epsilon$ used for L-BFGS-B and termination of L1CG is $5\cdot 10^{-4}$. In order to avoid large discrete optimization problems exceeding current optimization capabilities, in the event that another termination criterion is not met, we limit the maximum number of iterations for all methods to $T = 100$. The bit-depth for fixed-point discrete variables $\dot{\vc{w}}$ is chosen as 6, and their range is tentatively adjusted in successive iterations of TotalQBoost based on optimal values of the continuous variables $\vc{w}$ seen in L-BFGS-B solutions. The number of hot-starting columns $T'$ for experiments {\bf D.} and {\bf E.} is chosen on a per-data-set basis as shown in Appendix C.

\subsection{Results on Public Data Sets}
We perform the experiments listed in Sect.~\ref{expe_setup} on twelve public data sets commonly used for benchmarking of boosting algorithms. Due to the expensive nature of our discrete optimization, we can only afford one train-validation split of 80-20\% per data set. Because of that we do not have averaged test results, but we offer error rates produced on the validation set as indicators of generalization. A compressed representation of the results is shown in Fig.~\ref{fig:error_vs_card}. As a mock-up of how model selection is usually done via extensive cross-validation, when one or more of experiments {\bf A.}-{\bf B.} and {\bf C.}-{\bf E.} produce classifiers with the same cardinality for the same data set, we select a group representative according to the principle of Pareto efficiency and draw the resulting Pareto frontiers \citep{Kung75} respectively for L1CG/UCG and CP.

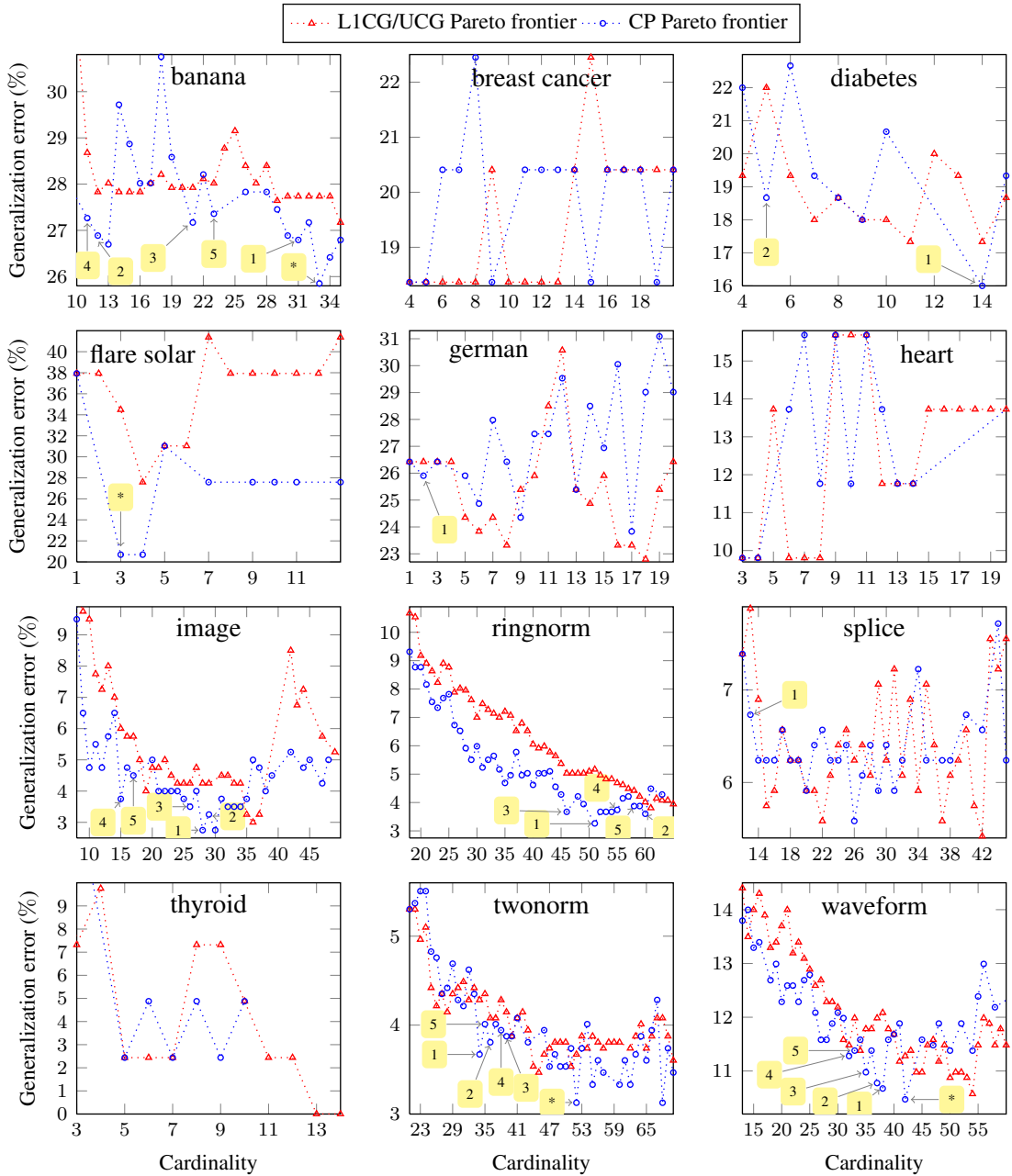
\begin{figure*}[t!]
\centering
\tikzset{every mark/.append style={solid,scale=0.7},every pin/.style={fill=yellow!50!white,rectangle,rounded corners=2pt,font=\tiny},arrow/.style={arrows=<->}}
\pgfplotsset{every axis/.append style={line width=0.5pt,height=4.9325cm,width=5.4cm,cycle list ={
{red,dotted,mark=triangle},
{blue,dotted,mark=o, mark options={solid,scale=0.55}}}}}
\begin{tikzpicture}                             
	\begin{axis} [name=plot1,legend columns=-1,legend entries={L1CG/UCG Pareto frontier, CP Pareto frontier},legend style={at={(0.78,1.05)},font=\footnotesize, anchor=south west},tick label style={font=\scriptsize},y label style={font=\footnotesize},ylabel=Generalization error ($\%$), xmin=10,xmax=35,ymin=25.8,ymax=30.8,xtick={10,13,...,50},ytick={26,27,...,33}]
	\node at (axis description cs:0.5,0.91) {banana};
	\node[coordinate,pin={[pin distance=0.4cm,pin edge={<-}]185:{1} } ] at (axis cs:30.6,26.75) {};
	\node[coordinate,pin={[pin distance=0.25cm,pin edge={<-}]300:{2} } ] at (axis cs:12.1,26.75) {};
	\node[coordinate,pin={[pin distance=0.4cm,pin edge={<-}]220:{3} } ] at (axis cs:20.5,27.1) {};
	\node[coordinate,pin={[pin distance=0.4cm,pin edge={<-}]270:{4} } ] at (axis cs:11,27.15) {};
	\node[coordinate,pin={[pin distance=0.3cm,pin edge={<-}]-90:{5} } ] at (axis cs:23,27.25) {};
	\node[coordinate,pin={[pin distance=0.12cm,pin edge={<-}]175:{*} } ] at (axis cs:32.6,25.88) {};
\addplot coordinates {(3,31.037700) (4,27.924500) (5,30.471700) (6,31.226400) (7,31.226400) (8,30.471700) (9,30.094300) (10,31.603800) (11,28.679200) (12,27.830200) (13,28.018900) (14,27.830200) (15,27.830200) (16,27.830200) (17,28.018900) (18,28.207500) (19,27.924500) (20,27.924500) (21,27.924500) (22,28.113200) (23,28.018900) (24,28.773600) (25,29.150900) (26,28.396200) (27,28.018900) (28,28.396200) (29,27.641500) (30,27.735800) (31,27.735800) (32,27.735800) (33,27.735800) (34,27.735800) (35,27.169800) (36,27.075500) (37,27.075500) (38,26.792500) (39,26.886800) (40,26.886800) (41,27.264200) (42,27.641500) (43,27.735800) (44,27.924500) (45,27.924500) (46,27.924500) (47,27.641500) (48,27.641500) (49,27.641500) (50,27.830200) };
\addplot coordinates {(3,28.773600) (4,27.924500) (6,31.226400) (9,28.113200) (11,27.264200) (12,26.886800) (13,26.698100) (14,29.717000) (15,28.867900) (16,28.018900) (17,28.018900) (18,30.754700) (19,28.584900) (21,27.169800) (22,28.207500) (23,27.358500) (26,27.830200) (28,27.830200) (29,27.452800) (30,26.886800) (31,26.792500) (32,27.169800) (33,25.849100) (34,26.415100) (35,26.792500) (36,26.698100) (37,26.226400) (38,26.037700) (39,27.169800) (40,28.113200) (43,27.547200) (44,26.981100) };
	\end{axis}
	
	\begin{axis} [name=plot2,at={($(plot1.east)+(1cm,0)$)},anchor=west,tick label style={font=\scriptsize}, xmin=4,xmax=20,ymin=18.3,ymax=22.5,xtick={4,6,...,18},ytick={18,19,...,29}]
	\node at (axis description cs:0.5,0.9) {breast cancer};
\addplot coordinates {(4,18.367300) (5,18.367300) (6,18.367300) (7,18.367300) (8,18.367300) (9,20.408200) (10,18.367300) (11,18.367300) (12,18.367300) (13,18.367300) (14,20.408200) (15,22.449000) (16,20.408200) (17,20.408200) (18,20.408200) (19,20.408200) (20,20.408200) (21,22.449000) (22,20.408200) (23,24.489800) (24,20.408200) (25,24.489800) };
\addplot coordinates {(4,18.367300) (5,18.367300) (6,20.408200) (7,20.408200) (8,22.449000) (9,18.367300) (11,20.408200) (12,20.408200) (13,20.408200) (14,20.408200) (15,18.367300) (16,20.408200) (17,20.408200) (18,20.408200) (19,18.367300) (20,20.408200) (21,22.449000) (22,22.449000) (23,24.489800) (25,22.449000) };
	\end{axis}
	
	\begin{axis} [name=plot3,at={($(plot2.east)+(1cm,0)$)},anchor=west,tick label style={font=\scriptsize}, xmin=4,xmax=15,ymin=16,ymax=23,xtick={4,6,...,14},ytick={16,17,...,22}]
	\node at (axis description cs:0.5,0.9) {diabetes};
	\node[coordinate,pin={[pin distance=0.5cm,pin edge={<-}]165:{1} } ] at (axis cs:13.7,16.1) {};
	\node[coordinate,pin={[pin distance=0.5cm,pin edge={<-}]-90:{2} } ] at (axis cs:5,18.5) {};
\addplot coordinates {(4,19.333300) (5,22.000000) (6,19.333300) (7,18.000000) (8,18.666700) (9,18.000000) (10,18.000000) (11,17.333300) (12,20.000000) (13,19.333300) (14,17.333300) (15,18.666700) (16,18.666700) (17,21.333300) (18,18.666700) (19,21.333300) (20,17.333300) (21,18.000000) (22,18.000000) (23,16.666700) (24,16.000000) (25,16.000000) (26,16.000000) (27,23.333300) (28,16.000000) (29,23.333300) (30,24.000000) };
\addplot coordinates {(4,22.000000) (5,18.666700) (6,22.666700) (7,19.333300) (8,18.666700) (9,18.000000) (10,20.666700) (14,16.000000) (15,19.333300) (16,20.000000) (17,19.333300) (18,22.000000) (23,21.333300) (29,25.333300) };
	\end{axis}
	
	\begin{axis} [name=plot4,at={($(plot1.south)-(0,0.65cm)$)},anchor=north,tick label style={font=\scriptsize},ylabel=Generalization error ($\%$),y label style={font=\footnotesize}, xmin=1,xmax=13,ymin=20,ymax=42,xtick={1,3,...,12},ytick={20,22,...,41}]
	\node at (axis description cs:0.25,0.9) {flare solar};
	\node[coordinate,pin={[pin distance=0.5cm,pin edge={<-}]90:{*} } ] at (axis cs:3,21.5) {};
\addplot coordinates {(1,37.931000) (2,37.931000) (3,34.482800) (4,27.586200) (5,31.034500) (6,31.034500) (7,41.379300) (8,37.931000) (9,37.931000) (10,37.931000) (11,37.931000) (12,37.931000) (13,41.379300) (14,44.827600) };
\addplot coordinates {(1,37.931000) (3,20.689700) (4,20.689700) (5,31.034500) (7,27.586200) (9,27.586200) (10,27.586200) (11,27.586200) (13,27.586200) };
	\end{axis}
	
	\begin{axis} [name=plot5,at={($(plot2.south)-(0,0.65cm)$)},anchor=north,y label style={font=\footnotesize},tick label style={font=\scriptsize}, xmin=1,xmax=20,ymin=22.7,ymax=31.3,xtick={1,3,...,50},ytick={21,22,...,33}]
	\node at (axis description cs:0.3,0.9) {german};
	\node[coordinate,pin={[pin distance=0.5cm,pin edge={<-}]-80:{1} } ] at (axis cs:2.15,25.7) {};
\addplot coordinates {(1,26.424900) (2,26.424900) (3,26.424900) (4,26.424900) (5,24.352300) (6,23.834200) (7,24.352300) (8,23.316100) (9,25.388600) (10,25.906700) (11,28.497400) (12,30.569900) (13,25.388600) (14,24.870500) (15,25.906700) (16,23.316100) (17,23.316100) (18,22.797900) (19,25.388600) (20,26.424900) (21,24.870500) (22,22.797900) (23,21.761700) (24,21.761700) (25,27.979300) (26,27.461100) (27,21.761700) (28,22.279800) (29,26.424900) (30,27.461100) (31,22.279800) (32,23.316100) (33,26.943000) (34,21.761700) (35,26.943000) (36,26.943000) (37,22.279800) (38,23.316100) (39,27.979300) (40,31.606200) (41,29.015500) (42,29.533700) (43,23.316100) (44,23.834200) (45,24.352300) (46,30.051800) (47,24.870500) (48,24.870500) (49,27.979300) (50,29.533700) };
\addplot coordinates {(1,26.424900) (2,25.906700) (3,26.424900) (5,25.906700) (6,24.870500) (7,27.979300) (8,26.424900) (9,24.352300) (10,27.461100) (11,27.461100) (12,29.533700) (13,25.388600) (14,28.497400) (15,26.943000) (16,30.051800) (17,23.834200) (18,29.015500) (19,31.088100) (20,29.015500) (21,30.051800) (22,31.088100) (23,28.497400) (24,29.533700) (25,29.015500) (26,30.569900) (27,30.051800) (28,28.497400) (29,31.606200) (30,30.569900) (31,28.497400) (32,28.497400) (33,28.497400) (34,28.497400) (35,27.979300) (36,26.943000) (37,28.497400) (39,26.943000) (41,32.124400) (42,29.533700) (45,31.088100) };
	\end{axis}
	
	\begin{axis} [name=plot6,at={($(plot3.south)-(0,0.65cm)$)},anchor=north, tick label style={font=\scriptsize}, xmin=3,xmax=20,ymin=9.7,ymax=15.8,xtick={3,5,...,35},ytick={9,10,...,22}]
	\node at (axis description cs:0.7,0.9) {heart};
\addplot coordinates {(3,9.803900) (4,9.803900) (5,13.725500) (6,9.803900) (7,9.803900) (8,9.803900) (9,15.686300) (10,15.686300) (11,15.686300) (12,11.764700) (13,11.764700) (14,11.764700) (15,13.725500) (16,13.725500) (17,13.725500) (18,13.725500) (19,13.725500) (20,13.725500) (21,15.686300) (22,15.686300) (23,17.647100) (24,15.686300) (25,17.647100) (26,15.686300) (27,15.686300) (29,17.647100) (30,15.686300) (32,15.686300) (33,15.686300) (35,15.686300) };
\addplot coordinates {(3,9.803900) (4,9.803900) (6,13.725500) (7,15.686300) (8,11.764700) (9,15.686300) (10,11.764700) (11,15.686300) (12,13.725500) (13,11.764700) (14,11.764700) (26,15.686300) (27,15.686300) (28,19.607800) (29,15.686300) (30,19.607800) (31,15.686300) };
	\end{axis}
	
	\begin{axis} [name=plot7,at={($(plot4.south)-(0,0.65cm)$)},anchor=north,ylabel=Generalization error ($\%$),y label style={font=\footnotesize}, tick label style={font=\scriptsize}, xmin=8,xmax=49.9,ymin=2.5,ymax=9.9,xtick={10,15,...,80},ytick={2,3,...,31}]
	\node at (axis description cs:0.5,0.9) {image};
	\node[coordinate,pin={[pin distance=0.1cm,pin edge={<-}]180:{1} } ] at (axis cs:27.3,2.75) {};
	\node[coordinate,pin={[pin distance=0.1cm,pin edge={<-}]-3:{2} } ] at (axis cs:29.6,3.2) {};
	\node[coordinate,pin={[pin distance=0.2cm,pin edge={<-}]180:{3} } ] at (axis cs:25.4,3.5) {};
	\node[coordinate,pin={[pin distance=0.1cm,pin edge={<-}]240:{4} } ] at (axis cs:14.8,3.65) {};
	\node[coordinate,pin={[pin distance=0.4cm,pin edge={<-}]270:{5} } ] at (axis cs:17,4.4) {};
\addplot coordinates {(4,13.250000) (5,13.250000) (6,13.000000) (7,11.000000) (8,10.250000) (9,9.750000) (10,9.500000) (11,7.750000) (12,7.250000) (13,8.000000) (14,7.000000) (15,6.000000) (16,5.750000) (17,5.750000) (18,5.000000) (19,4.000000) (20,4.750000) (21,4.750000) (22,5.000000) (23,4.500000) (24,4.250000) (25,4.250000) (26,4.250000) (27,4.750000) (28,4.250000) (29,4.250000) (31,4.500000) (32,4.500000) (33,4.250000) (34,4.250000) (35,3.250000) (36,3.000000) (37,3.250000) (42,8.500000) (43,6.750000) (44,7.250000) (47,5.750000) (49,5.250000) (53,4.750000) (54,3.750000) (55,3.500000) (58,3.500000) (62,3.500000) (63,3.500000) (71,3.000000) (73,2.750000) (75,2.250000) (79,2.500000) (80,2.500000) };
\addplot coordinates {(6,15.500000) (8,9.500000) (9,6.500000) (10,4.750000) (11,5.500000) (12,4.750000) (13,5.750000) (14,6.500000) (15,3.750000) (16,4.750000) (17,4.500000) (20,5.000000) (21,4.000000) (22,4.000000) (23,4.000000) (24,4.000000) (25,3.750000) (26,3.500000) (27,4.000000) (28,2.750000) (29,3.250000) (30,2.750000) (31,3.750000) (32,3.500000) (33,3.500000) (34,3.500000) (35,3.750000) (36,5.000000) (37,4.750000) (38,4.000000) (39,4.500000) (42,5.250000) (44,4.750000) (45,5.000000) (47,4.250000) (48,5.000000) };
	\end{axis}
	
	\begin{axis} [name=plot8,at={($(plot5.south)-(0,0.65cm)$)},anchor=north,y label style={font=\footnotesize}, tick label style={font=\scriptsize}, xmin=18,xmax=65,ymin=2.75,ymax=10.9,xtick={5,10,...,60},ytick={2,3,...,40}]
	\node at (axis description cs:0.5,0.9) {ringnorm};
	\node[coordinate,pin={[pin distance=0.6cm,pin edge={<-}]180:{1} } ] at (axis cs:50.4,3.25) {};
	\node[coordinate,pin={[pin distance=0.08cm,pin edge={<-}]-5:{2} } ] at (axis cs:60.3,3.55) {};
	\node[coordinate,pin={[pin distance=0.6cm,pin edge={<-}]178:{3} } ] at (axis cs:45,3.67) {};
	\node[coordinate,pin={[pin distance=0.1cm,pin edge={<-}]130:{4} } ] at (axis cs:54.5,3.82) {};
	\node[coordinate,pin={[pin distance=0.1cm,pin edge={<-}]-100:{5} } ] at (axis cs:57.7,3.8) {};
\addplot coordinates {(5,29.503700) (6,23.997300) (7,22.025800) (8,20.666200) (9,20.326300) (10,19.578500) (11,16.927300) (12,15.975500) (13,15.091800) (14,14.955800) (15,14.004100) (16,12.440500) (17,10.673000) (18,10.673000) (19,10.537000) (20,9.177400) (21,8.905500) (22,8.633600) (23,8.225700) (24,8.905500) (25,8.769500) (26,7.885800) (27,8.021800) (28,7.953800) (29,7.613900) (30,7.002000) (31,7.477900) (32,7.274000) (33,7.138000) (34,7.002000) (35,7.206000) (36,7.070000) (37,6.526200) (38,6.798100) (39,6.526200) (40,6.050300) (41,5.914300) (42,5.982300) (43,5.778400) (44,5.642400) (45,5.370500) (46,5.030600) (47,5.030600) (48,5.030600) (49,5.030600) (50,5.098600) (51,5.166600) (52,4.962600) (53,4.826600) (54,4.826600) (55,4.690700) (56,4.622700) (57,4.486700) (58,4.418800) (59,4.214800) (60,4.010900) (61,3.806900) (62,4.146800) (63,4.078900) (64,4.078900) (65,3.942900) (66,3.874900) (67,3.942900) (68,3.806900) (69,3.739000) (70,3.739000) (71,3.671000) (72,9.857200) (73,3.535000) (74,3.467000) (75,3.195100) (76,3.467000) (77,3.331100) (78,3.467000) (79,3.331100) (80,3.195100) (81,3.127100) (82,2.923200) (83,3.127100) (85,3.127100) };
\addplot coordinates {(9,19.102700) (11,16.859300) (12,15.635600) (13,13.800100) (14,12.712400) (15,11.488800) (16,11.352800) (17,11.148900) (18,9.313400) (19,8.769500) (20,8.769500) (21,8.157700) (22,7.545900) (23,7.341900) (24,7.681800) (25,7.817800) (26,6.730100) (27,6.526200) (28,5.914300) (29,5.506500) (30,5.982300) (31,5.234500) (32,5.506500) (33,5.642400) (34,5.166600) (35,4.690700) (36,4.962600) (37,5.778400) (38,4.962600) (39,5.030600) (40,4.622700) (41,5.030600) (42,5.030600) (43,5.098600) (44,4.554700) (45,4.282800) (46,3.671000) (48,4.214800) (49,3.942900) (51,3.263100) (52,3.671000) (53,3.671000) (54,3.671000) (55,3.739000) (56,4.146800) (57,4.214800) (58,3.874900) (59,3.874900) (60,3.603000) (61,4.486700) (63,4.282800) };
	\end{axis}
	
		\begin{axis} [name=plot9,at={($(plot6.south)-(0,0.65cm)$)},anchor=north,x label style={font=\footnotesize}, y label style={font=\footnotesize}, tick label style={font=\scriptsize}, xmin=12,xmax=45,ymin=5.4,ymax=7.9,xtick={10,14,...,65},ytick={5,6,...,12}]
	\node at (axis description cs:0.5,0.9) {splice};
	\node[coordinate,pin={[pin distance=0.4cm,pin edge={<-}]10:{1} } ] at (axis cs:13.5,6.74) {};
\addplot coordinates {(6,10.180600) (7,9.688000) (8,10.344800) (9,9.195400) (10,6.896600) (11,8.374400) (12,7.389200) (13,7.881800) (14,6.896600) (15,5.747100) (16,5.911300) (17,6.568100) (18,6.239700) (19,6.239700) (20,5.911300) (21,5.911300) (22,5.582900) (23,6.075500) (24,6.403900) (25,6.568100) (26,6.239700) (27,6.403900) (28,6.075500) (29,7.060800) (30,6.239700) (31,7.225000) (32,6.075500) (33,6.896600) (34,5.911300) (35,7.060800) (36,6.403900) (37,5.582900) (38,6.075500) (39,6.239700) (40,6.568100) (41,5.747100) (42,5.418700) (43,7.553400) (44,7.225000) (45,7.553400) (46,7.225000) (47,7.060800) (48,7.060800) (49,7.060800) (50,5.418700) (51,7.225000) (52,7.717600) (53,7.717600) (54,7.881800) (55,7.717600) (56,8.046000) (57,7.717600) (58,8.210200) (59,5.418700) (60,7.389200) (61,7.225000) (62,5.418700) (63,7.060800) (64,7.225000) (65,5.582900) };
\addplot coordinates {(6,11.494300) (7,11.165800) (8,10.344800) (9,9.359600) (10,7.553400) (11,8.374400) (12,7.389200) (13,6.732300) (14,6.239700) (15,6.239700) (16,6.239700) (17,6.568100) (18,6.239700) (19,6.239700) (20,5.911300) (21,6.403900) (22,6.568100) (23,6.239700) (24,6.239700) (25,6.403900) (26,5.582900) (27,6.075500) (28,6.403900) (29,5.911300) (30,6.403900) (31,5.911300) (32,6.239700) (34,7.225000) (35,6.239700) (37,6.239700) (38,6.239700) (40,6.732300) (42,6.568100) (44,7.717600) (45,6.239700) (48,7.060800) (49,7.225000) (50,6.732300) (51,7.225000) (53,6.896600) (54,6.896600) (55,7.389200) (56,7.060800) };
	\end{axis}
	
		\begin{axis} [name=plot10,at={($(plot7.south)-(0,0.65cm)$)},anchor=north,x label style={font=\footnotesize},ylabel=Generalization error ($\%$),y label style={font=\footnotesize}, tick label style={font=\scriptsize}, xlabel=Cardinality, xmin=3,xmax=14,ymin=0,ymax=10,xtick={1,3,...,13},ytick={0,1,...,9}]
	\node at (axis description cs:0.5,0.9) {thyroid};
\addplot coordinates {(1,19.512200) (2,19.512200) (3,7.317100) (4,9.756100) (5,2.439000) (6,2.439000) (7,2.439000) (8,7.317100) (9,7.317100) (10,4.878000) (11,2.439000) (12,2.439000) (13,0.000000) (14,0.000000) (15,0.000000) (16,0.000000) (17,0.000000) (18,0.000000) (19,0.000000) (20,0.000000) (21,2.439000) };
\addplot coordinates {(2,19.512200) (5,2.439000) (6,4.878000) (7,2.439000) (8,4.878000) (9,2.439000) (10,4.878000) };
	\end{axis}
	
	\begin{axis} [name=plot11,at={($(plot8.south)-(0,0.65cm)$)},anchor=north,x label style={font=\footnotesize},y label style={font=\footnotesize}, tick label style={font=\scriptsize}, xlabel=Cardinality, xmin=21,xmax=70,ymin=3,ymax=5.6,xtick={5,11,...,75},ytick={2,3,...,25}]
	\node at (axis description cs:0.5,0.9) {twonorm};
	\node[coordinate,pin={[pin distance=0.4cm,pin edge={<-}]180:{1} } ] at (axis cs:33.2,3.67) {};
	\node[coordinate,pin={[pin distance=0.5cm,pin edge={<-}]-100:{2} } ] at (axis cs:35.8,3.77) {};
	\node[coordinate,pin={[pin distance=0.5cm,pin edge={<-}]-80:{3} } ] at (axis cs:39.1,3.84) {};
	\node[coordinate,pin={[pin distance=0.5cm,pin edge={<-}]-90:{4} } ] at (axis cs:38,3.91) {};
	\node[coordinate,pin={[pin distance=0.5cm,pin edge={<-}]180:{5} } ] at (axis cs:34.3,4.01) {};
	\node[coordinate,pin={[pin distance=0.1cm,pin edge={<-}]180:{*} } ] at (axis cs:51.2,3.13) {};
\addplot coordinates {(5,21.074100) (6,21.210100) (7,15.839600) (8,15.907500) (9,13.256300) (10,13.664200) (11,11.420800) (12,11.624700) (13,9.517300) (14,10.197100) (15,8.021800) (16,8.361700) (17,6.594200) (18,6.390200) (19,5.234500) (20,5.370500) (21,5.302500) (22,5.302500) (23,4.962600) (24,5.098600) (25,4.418800) (26,4.214800) (27,4.350800) (28,4.146800) (29,4.350800) (30,4.418800) (31,4.486700) (32,4.282800) (33,4.418800) (34,4.282800) (35,4.350800) (36,4.078900) (37,4.078900) (38,4.282800) (39,4.146800) (40,3.874900) (41,4.078900) (42,4.146800) (43,3.942900) (44,3.535000) (45,3.467000) (46,3.671000) (47,3.739000) (48,3.806900) (49,3.806900) (50,3.806900) (51,3.535000) (52,3.671000) (53,3.874900) (54,3.739000) (55,3.874900) (56,3.806900) (57,3.739000) (58,3.806900) (59,3.806900) (60,3.806900) (62,3.739000) (63,3.874900) (64,4.010900) (65,3.739000) (66,3.874900) (67,4.078900) (68,4.078900) (69,3.874900) (70,3.603000) (71,3.739000) (72,3.874900) (73,3.942900) (74,3.874900) (75,3.942900) };
\addplot coordinates {(5,21.074100) (6,20.598200) (7,15.839600) (8,15.907500) (9,13.256300) (10,15.227700) (11,10.469100) (12,11.216900) (13,9.517300) (14,10.197100) (15,9.517300) (16,8.361700) (17,6.594200) (18,6.186300) (19,5.234500) (20,6.118300) (21,5.302500) (22,5.370500) (23,5.506500) (24,5.506500) (25,4.826600) (26,4.758700) (27,4.350800) (28,4.418800) (29,4.690700) (30,4.282800) (31,4.214800) (32,4.622700) (33,4.350800) (34,3.671000) (35,4.010900) (36,3.806900) (37,4.010900) (38,3.942900) (39,3.874900) (40,3.874900) (41,4.078900) (43,3.806900) (46,3.942900) (47,3.535000) (48,3.671000) (49,3.535000) (50,3.535000) (51,3.739000) (52,3.127100) (53,3.739000) (54,4.010900) (55,3.331100) (56,3.603000) (57,3.467000) (60,3.331100) (61,3.603000) (62,3.331100) (63,3.671000) (64,3.874900) (65,3.603000) (66,3.942900) (67,4.282800) (68,3.127100) (69,3.739000) (70,3.467000) (71,3.603000) (73,3.806900) };
	\end{axis}
	
		\begin{axis} [name=plot12,at={($(plot9.south)-(0,0.65cm)$)},anchor=north,x label style={font=\footnotesize},y label style={font=\footnotesize}, tick label style={font=\scriptsize}, xlabel=Cardinality, xmin=13,xmax=60,ymin=10.2,ymax=14.5,xtick={10,15,...,59},ytick={10,11,...,19}]
	\node at (axis description cs:0.5,0.9) {waveform};
	\node[coordinate,pin={[pin distance=0.1cm,pin edge={<-}]190:{1} } ] at (axis cs:37.4,10.65) {};
	\node[coordinate,pin={[pin distance=0.5cm,pin edge={<-}]184:{2} } ] at (axis cs:36.2,10.75) {};
	\node[coordinate,pin={[pin distance=0.8cm,pin edge={<-}]184:{3} } ] at (axis cs:34.2,10.96) {};
	\node[coordinate,pin={[pin distance=0.9cm,pin edge={<-}]184:{4} } ] at (axis cs:31.2,11.28) {};
	\node[coordinate,pin={[pin distance=0.6cm,pin edge={<-}]180:{5} } ] at (axis cs:32.2,11.38) {};
	\node[coordinate,pin={[pin distance=0.4cm,pin edge={<-}]0:{*} } ] at (axis cs:43.05,10.47) {};
\addplot coordinates {(6,16.213500) (7,16.314200) (8,15.105700) (9,15.105700) (10,14.803600) (11,14.098700) (12,14.199400) (13,14.400800) (14,13.494500) (15,13.998000) (16,14.300100) (17,13.897300) (18,13.293100) (19,13.393800) (20,13.695900) (21,13.998000) (22,13.192300) (23,13.393800) (24,13.091600) (25,12.890200) (26,12.588100) (27,12.688800) (28,12.286000) (29,12.286000) (30,12.185300) (31,11.581100) (32,11.480400) (33,11.983900) (34,11.379700) (35,11.782500) (36,11.782500) (37,11.983900) (38,12.084600) (39,11.782500) (40,11.681800) (41,11.178200) (42,11.279000) (43,11.379700) (44,10.976800) (45,10.976800) (46,11.480400) (47,11.581100) (48,11.178200) (49,11.480400) (50,10.876100) (51,10.976800) (52,10.976800) (53,10.876100) (54,10.574000) (55,11.480400) (56,11.983900) (57,11.883200) (58,11.480400) (59,11.782500) (60,11.480400) (61,11.883200) (62,11.883200) (63,11.782500) (64,11.883200) (65,11.581100) (66,11.782500) (67,12.084600) (68,11.782500) (69,11.681800) (70,11.782500) (71,11.581100) (72,11.279000) (73,11.681800) (74,11.581100) (75,11.983900) (76,12.084600) (77,12.084600) (78,12.286000) (79,12.185300) (80,12.084600) };
\addplot coordinates {(6,16.918400) (7,15.710000) (8,15.508600) (9,13.998000) (10,14.602200) (11,13.494500) (12,15.005000) (13,13.796600) (14,13.998000) (15,13.293100) (16,13.393800) (18,12.688800) (19,12.990900) (20,12.286000) (21,12.588100) (22,12.588100) (23,12.286000) (24,12.688800) (25,12.789500) (26,12.084600) (27,11.581100) (28,11.581100) (29,11.883200) (30,12.084600) (31,11.983900) (32,11.279000) (33,11.379700) (34,11.581100) (35,10.976800) (36,11.379700) (37,10.775400) (38,10.674700) (39,11.581100) (40,11.681800) (41,11.883200) (42,10.473300) (45,11.581100) (47,11.480400) (48,11.883200) (50,11.379700) (52,11.883200) (54,11.379700) (55,12.386700) (56,12.990900) (58,12.185300) (61,12.386700) (62,12.386700) (63,13.091600) (64,12.990900) (65,12.688800) (66,12.084600) (69,12.286000) (71,12.286000) (72,12.588100) (73,12.487400) (75,12.487400) (76,12.789500) (77,13.293100) (78,12.185300) };
	\end{axis}
\end{tikzpicture}
 \caption{Generalization results on twelve public data sets commonly used for benchmarking of boosting algorithms. For visualization clarity the axes limits are adjusted to discard uninteresting regions of severe under- or overfitting. Additionally, the Pareto frontiers of L1CG/UCG- and CP-generated points are drawn to outline the model selection that can be done at different cardinalities. We also tag specific points on the CP Pareto frontier whose advantages over L1CG/UCG points are quantified in Table~\ref{table:sparsity_gains}.}
\label{fig:error_vs_card}
\end{figure*}

Table~\ref{table:sparsity_gains} shows a sampling of the most notable sparsity and generalization gains provided by CP. We define \emph{sparsity gain} as the improvement in sparsity given by a CP point on the Pareto frontier relative to the sparsest L1CG/UCG point with comparable generalization performance. The results show various sparsity gains between $2.5\%$ and $67.65\%$. Whenever there is no L1CG/UCG point to which to relate for sparsity gain, i.e. all L1CG/UCG points have worse generalization performance, we quantify \emph{generalization gain} as the improvement in generalization given by a CP point relative to the best L1CG/UCG point across all seen cardinalities. The results show various generalization gains between $0.95\%$ and $25\%$.

To compare the optimization success of experiment {\bf E.} against {\bf B.}, we counted the number of times {\bf E.} is worse than {\bf B.} at coinciding cardinalities and divided by the total number of times {\bf E.} and {\bf B} had coinciding cardinalities. This measurement yielded only $1.61\%$ on empirical risk and $2.05\%$ on training error, which confirms the view of early stopping as simply unoptimized CP. We did the same comparison for experiments {\bf D.} and {\bf E}. On empirical risk, {\bf D.} lost from {\bf E.} $23.81\%$ of the time and on training error $13.01\%$ of the time. We attribute these numbers mainly to our inability at present to fully and reliably optimize the generated CP problems.

\begin{table*}[t!]
\caption{Values in each entry are: generalization error of CP-generated point lying on Pareto frontier (\%), its cardinality, and its sparsity gain ({\bf \%}) relative to the sparsest L1CG/UCG point of comparable generalization. The rows are denoted by IDs that correspond to tags in Fig.~\ref{fig:error_vs_card}. On the data sets that have entries for *, CP also achieves generalization unattainable by the L1CG/UCG methods at any cardinalities up to 100. For such CP points we cannot compute well-defined sparsity gains, but we point out their generalization improvements relative to the best L1CG/UCG points across all seen cardinalities: {\bf banana 3.52\%, twonorm 9.8\%, waveform 0.95\%, flare solar 25\%}. For these and other points it is possible also to compute relative generalization gains on comparable cardinalities, but we leave that to visual inspection of Fig.~\ref{fig:error_vs_card}. breast cancer, heart, and thyroid do not have any CP points with sparsity gains. diabetes, german, and splice each have only one or two CP points with sparsity gains. The rest of the data sets have numerous CP points with sparsity gains, out of which we pick the top five by generalization error to list here.}
\centering
\small
\begin{tabular}{cccccc}
\toprule
ID & banana & image & ringnorm & twonorm & waveform\\
\midrule
\rowcolor[gray]{0.9}
* & 25.85, 33, {\bf N/A} & N/A & N/A & 3.10, 52, {\bf N/A} & 10.47, 42, {\bf N/A}\\
1 & 26.79, 31, {\bf 18.42} & 2.75, 28, {\bf 61.64} & 3.26, 51, {\bf 31.08} & 3.67, 34, {\bf 20.93} & 10.67, 38, {\bf 28.30}\\
\rowcolor[gray]{0.9}
2 & 26.89, 12, {\bf 67.57} & 3.25, 29, {\bf 17.14} & 3.60, 60, {\bf 16.67} & 3.81, 36, {\bf 16.28} & 10.78, 37, {\bf 30.19}\\
3 & 27.17, 21, {\bf 40.00} & 3.50, 26, {\bf 23.53} & 3.67, 46, {\bf 35.21} & 3.87, 39, {\bf 2.50}   & 10.98, 35, {\bf 20.45}\\
\rowcolor[gray]{0.9}
4 & 27.26, 11, {\bf 67.65} & 3.75, 15, {\bf 55.88} & 3.74, 55, {\bf 20.29} & 3.94, 38, {\bf 2.56}   & 11.28, 32, {\bf 20.00}\\
5 & 27.36, 23, {\bf 32.35} & 4.50, 17, {\bf 5.56}   & 3.87, 58, {\bf 3.33} &    4.01, 35, {\bf 10.26} & 11.38, 33, {\bf 2.94}\\
\toprule
ID & diabetes & german & splice & flare solar\\
\midrule
\rowcolor[gray]{0.9}
* & N/A & N/A & N/A & 20.95, 3, {\bf N/A}\\
1 & 16.00, 14, {\bf 41.67} & 25.91, 2, {\bf 60.00} & 6.73, 13, {\bf 7.14} & N/A\\
\rowcolor[gray]{0.9}
2 & 18.67, 5, {\bf 16.67} & N/A & N/A & N/A\\
\rowcolor[gray]{0.9}
\bottomrule
\end{tabular}\label{table:sparsity_gains}
\end{table*}

\subsection{Results on Project Glass}\label{glass}
Lastly, we saw a significant impact of CP also on boosted cascades for the eye gesture detector used in Google's Project Glass. The wearable device imposes severe memory, energy, and processing restrictions, so the trained detector is required to be maximally compact at any desired detection accuracy. In comparisons with corrective and totally corrective versions of convex boosting with early stopping, we observed that CP optimizations on cascades of cardinality ${\sim}20$ using training sets of ${\sim}10^5$ examples and ${\sim}10^4$ features can reduce false detections by $15$-$45\%$ at fixed recall.

\section{Conclusion}\label{conclusion}
Historically, in comparisons between AdaBoost and explicitly regularized boosting algorithms, e.g. \citet{Duchi09}, what we call here UCG with early stopping has been known to yield better results than L1CG. While it has been somewhat of a mystery why a much simplified algorithm such as AdaBoost is performing well, in the context of cardinality-penalized boosting we concluded that early stopping is nothing but suboptimal cardinality regularization. Our experiments with explicitly cardinality-penalized CG indicate that there is room for significant improvement as we attempt to solve the combinatorial problems arising from CP. Because in this work we used only a classical heuristic algorithm in place of quantum hardware, it remains to be seen how much better results can be obtained when the truly intended optimization engine for this learning algorithm becomes widely available.

\section*{Acknowledgements}
We thank Manfred Warmuth, Corinna Cortes, Christian Szegedy, James Philbin, Bo Wu, and Hayes Raffle for useful discussions and comments on early drafts of the paper.

\bibliography{tqboost}

\appendix
\section{Primal CP and the Lagrange Dual: Indirect Argument}\label{indirect}
We can show that the Lagrange dual function is unaffected by the cardinality term in the primal without introducing the dummy primal variables $\vc{\eta}$ used in Subsection~\ref{card_penal}:
\begin{equation}
\min_{\vc{w},\vc{\gamma}} \sum_{i=1}^m l(\gamma_i) + \nu\vc{1}^\top\vc{w} + \lambda \bm{\hbox{card}}(\vc{w}) \label{primal_with_l0_no_dummy}
\enspace \mbox{ s.t. } \gamma_i = y_i \mathcal{H}_{i:} \vc{w} \mbox{ (for } i = 1,\ldots,m\mbox{)}, \vc{w} \succeq \vc{0}
\end{equation}

Let the Lagrangian without the term $\lambda \bm{\hbox{card}}(\vc{w})$ be denoted by $L(\vc{w}, \vc{\gamma}, \vc{u}, \vc{p})$, so that the Lagrangian of \eqref{primal_with_l0_no_dummy} is $L(\vc{w}, \vc{\gamma}, \vc{u}, \vc{p}) + \lambda \bm{\hbox{card}}(\vc{w})$. Then the dual functions of primal without $\lambda \bm{\hbox{card}}(\vc{w})$ and with it are denoted respectively by
\begin{equation}
F(\vc{u},\vc{p}) = \inf_{\vc{w},\vc{\gamma}} L(\vc{w}, \vc{\gamma}, \vc{u}, \vc{p})
\end{equation} and
\begin{equation}
F_{\lambda}(\vc{u},\vc{p}) = \inf_{\vc{w},\vc{\gamma}} L(\vc{w}, \vc{\gamma}, \vc{u}, \vc{p}) + \lambda \bm{\hbox{card}}(\vc{w}).
\end{equation}
A straightforward observation is that $L(\vc{w}, \vc{\gamma}, \vc{u}, \vc{p}) \leq L(\vc{w}, \vc{\gamma}, \vc{u}, \vc{p}) + \lambda \bm{\hbox{card}}(\vc{w})$ with equality holding only for $\vc{w} = \vc{0}$. We use this in the following:
\begin{theorem}\label{shin_argument}
For any $(\vc{u}, \vc{p})$ at which $F$ and $F_{\lambda}$ are finite, $F = F_{\lambda}$.
\end{theorem}
\begin{proof}
The Lagrangian without $\lambda \bm{\hbox{card}}(\vc{w})$ is linear in $\vc{w}$:
\begin{equation*}
L(\vc{w}, \vc{\gamma}, \vc{u}, \vc{p})
=\underbrace{\left(\nu \vc{1}^\top - \vc{u}^\top \mbox{diag}\left(\vc{y}\right)\mathcal{H} - \vc{p}^\top\right)}_{\vc{r}^\top}\vc{w} + \vc{u}^\top\vc{\gamma} + \sum_{i=1}^m l\left(\gamma_i\right) \enspace \mbox{, s.t. } \vc{p} \succeq \vc{0} \enspace. 
\end{equation*}
For $\vc{r}^\top \neq 0$, we have $F(\vc{u},\vc{p}) = -\infty$ and $F_{\lambda}(\vc{u},\vc{p}) = -\infty$. For $\vc{r}^\top = \vc{0}$, we see that $\vc{w} = \vc{0}$ minimizes both $L(\vc{w}, \vc{\gamma}, \vc{u}, \vc{p})$ and $L(\vc{w}, \vc{\gamma}, \vc{u}, \vc{p}) + \lambda \bm{\hbox{card}}(\vc{w})$:
\begin{equation*}
F_{\lambda}(\vc{u},\vc{p})= \inf_{\vc{w},\vc{\gamma}} \vc{0}^\top\vc{w} + \vc{u}^\top\vc{\gamma} + \sum_{i=1}^m l\left(\gamma_i\right) + \lambda \bm{\hbox{card}}(\vc{w})\\=  \inf_{\vc{\gamma}} \vc{u}^\top\vc{\gamma} + \sum_{i=1}^m l\left(\gamma_i\right)= F(\vc{u},\vc{p}) \enspace .
\end{equation*}
\end{proof}

\section{Primal CP and the Lagrange Dual: Intuitive Argument}\label{intuitive}
Intuitively, the cardinality term could potentially affect the dual function only through its gradient. However, because $\bm{\hbox{card}}(\vc{w})$ contains no gradient information, no influence from it can appear in the dual. This is properly understood by studying CP within the realm of distribution theory and generalized functions \citep{Lighthill} as the limit of a sequence:
\begin{equation}
\bm{\hbox{card}}(\vc{w}) = \lim_{q \rightarrow \infty} \Omega_{q,k}(\vc{w}) \enspace ,
\end{equation}
where, for any $k > 1$, $\displaystyle \Omega_{q,k}(\vc{w}) = \sum_{i=1}^n \left(\abs{w_i} + q^{-1}\right)^{q^{-k}}$.

The functions $\Omega_{q,k}(\vc{w})$ for finite $q$ are differentiable everywhere. Moreover, with $k > 1$, differentiability as understood within the theory of generalized functions is maintained for $q \rightarrow \infty$ as well. According to the definition of CP, we naturally expect $\frac{\partial \bm{card}(\vc{w})}{\partial w_i} = 0$ for $w_i \neq 0$. At first glance, points with $w_i = 0$ are problematic from the perspective of differentiability. However, in the limit of $q \rightarrow \infty$, $\frac{\partial \Omega_{q,k}(\vc{w})}{\partial w_i}\big|_{w_i = 0}$ is still zero for all defined $k$:
\begin{theorem}\label{thm:l0_gradient}
The gradient of $\bm{\hbox{card}}(\vc{w})$ is $\enspace \displaystyle \frac{d\bm{\hbox{card}}(\vc{w})}{d\vc{w}} = \lim_{q \rightarrow \infty} \frac{d \Omega_{q,k}(\vc{w})}{d \vc{w}} = \vc{0} \enspace$.
\end{theorem}
\begin{proof}
The first equality follows from \emph{Definition 6} in \citep{Lighthill}. To prove the second equality, the derivative of $\Omega_{q,k}(\vc{w})$ with respect to $w_i$ is
\begin{equation}
\partial \Omega_{q,k}(\vc{w})/\partial w_i = q^{-k}\left(\abs{w_i} + q^{-1}\right)^{q^{-k}-1}\sign(w_i).
\end{equation}
Hence, we need to show that $\enspace \displaystyle \lim_{q \rightarrow \infty} \frac{q^{-k}}{\left(\abs{w_i} + q^{-1}\right)^{1-q^{-k}}} = 0 \enspace $,
which is obviously true for $w_i \neq 0$. For $w_i = 0$, $k > 1$ results in $\enspace \displaystyle \lim_{q \rightarrow \infty} \frac{q^{-k}}{q^{q^{-k} - 1}} = \lim_{q \rightarrow \infty} \frac{1}{q^{q^{-k} + k - 1}}  = 0 \displaystyle$.
\end{proof}

\section{Hot-starting Columns for Experiments D. and E.}
\begin{table*}[h!]
\caption{Numbers of hot-starting columns used in experiments {\bf D.} and {\bf E.}}
\vspace{0.3cm}
\includegraphics[width=15cm]{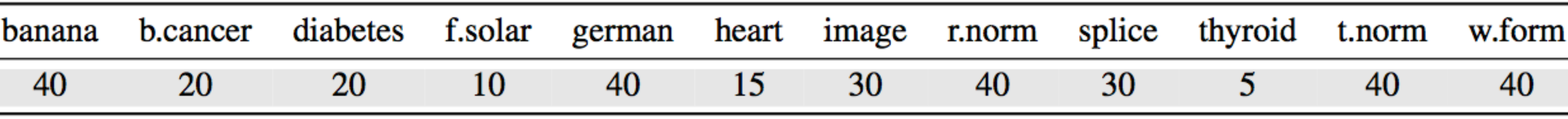}
\end{table*}

\end{document}